\documentclass[runningheads]{llncs}
\usepackage{helvet}
\usepackage{courier}
\usepackage{url}
\usepackage{graphicx}
\usepackage{multirow, subfig, calc}
\usepackage{amsmath}
\usepackage{mathtools}
\usepackage{booktabs}
\usepackage{supertabular}
\usepackage{amsfonts}
\usepackage{verbatim}
\usepackage{xr}
\usepackage{caption}
\usepackage[misc]{ifsym}

\newtheorem{thm}{Theorem}

\begin{document}
\title{Towards Robust Neural Networks with \\ Lipschitz Continuity}

\titlerunning{Towards Robust Neural Networks with Lipschitz Continuity}
%
\author{Muhammad Usama \and Dong Eui Chang \Letter}

%
\authorrunning{M. Usama and D. E. Chang}
%
\institute{
School of Electrical Engineering \\
Korea Advanced Institute of Science and Technology \\
Daejeon, Republic of Korea \\
\email{\{usama,dechang\}@kaist.ac.kr}}

\maketitle              
\begin{abstract}
Deep neural networks have shown remarkable performance across a wide range of vision-based tasks, particularly due to the availability of large-scale datasets for training and better architectures. However, data seen in the real world are often affected by distortions that not accounted for by the training datasets. In this paper, we address the challenge of robustness and stability of neural networks and propose a general training method that can be used to make the existing neural network architectures more robust and stable to input visual perturbations while using only available datasets for training. Proposed training method is convenient to use as it does not require data augmentation or changes in the network architecture. We provide theoretical proof as well as empirical evidence for the efficiency of the proposed training method by performing experiments with existing neural network architectures and demonstrate that same architecture when trained with the proposed training method perform better than when trained with conventional training approach in the presence of noisy datasets.

\keywords{deep neural networks  \and robust neural networks \and Lipschitz continuity.}
\end{abstract}
\section{Introduction}
\label{Introduction}
Recent advances in deep learning have immensely increased the representational capabilities of the neural networks and made them powerful enough to be applied to different vision-based tasks including image classification \cite{2,3,4,22}, object detection \cite{5,6}, image captioning \cite{31} as well as to deep reinforcement learning \cite{8,10}. Some important factors that explain the rapid development of deep learning include emergence of dedicated mathematical frameworks for deep neural networks \cite{mathematical_framework}, availability of large scale annotated datasets \cite{15,16}, improvements in the network architectures \cite{4,17} and open source deep learning libraries \cite{13,14}.

\begin{table*}[t]
	\caption{Effect of input image quality on the deep learning model prediction. We trained resnet-20 architecture with standard and proposed training procedure and tested them on a CIFAR-10 dataset image. Model trained with standard method fails to correctly classify the image as the severity of distortion increases while that trained with proposed method correctly classifies all images with high confidence. }
	\label{paperSummary}
	\center
	\begin{tabular}{l p{2cm} c c c c}
	\toprule

	\multicolumn{2}{c}{Gaussian Noise std} & $\sigma=0.0$ & $\sigma=0.2$ & $\sigma=0.4$	& $  	\sigma=0.6$\\

	\midrule

	\multicolumn{2}{c}{Input to the model} & \includegraphics[width=1.7cm]{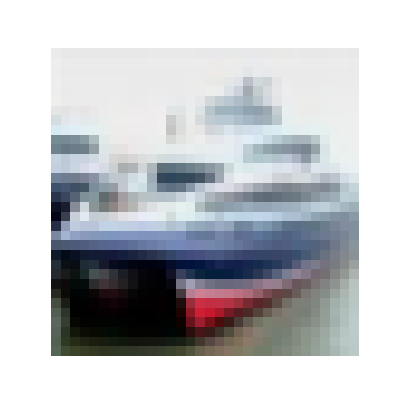} & 		    \includegraphics[width=1.7cm]{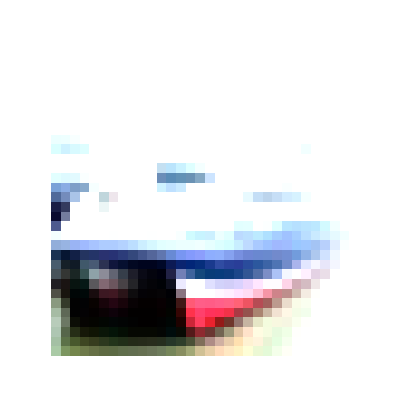} & \includegraphics[width=1.7cm]{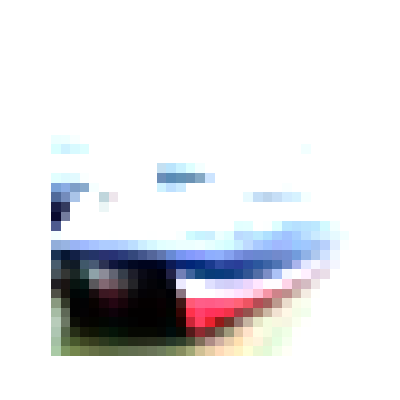} &     	    \includegraphics[width=1.7cm]{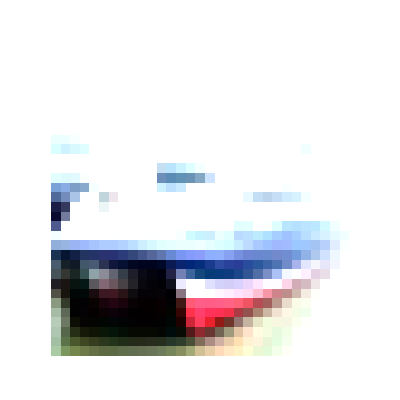}  \\

	\midrule
	
	\multirow{2}[8]{*}{standard training}& Model Prediction & ship & ship & bird & bird\\
	\cmidrule(l){2-6}
	& confidence for `ship' & 0.9999 & 0.5608 & 0.1266 & 0.0252\\
	
	\midrule
	
	\multirow{2}[8]{*}{proposed training}& Model Prediction & ship & ship & ship & ship \\ \cmidrule(l){2-6} & confidence        for `ship' & 0.9999 & 0.9986 & 0.8710 & 0.7215\\
	
	\bottomrule
	\end{tabular}
\end{table*}

Availability of large amounts of high-quality and distortionless image data is often assumed and the visual quality of training images is often overlooked while designing deep learning based applications. It has been shown that models trained with \textit{clean} data suffer with depreciation in their performance when tested on samples that are distorted with blur or noise distortions \cite{18,19}. In most real-world applications, the images undergo various forms of distortions owing to formatting, compression and post-processing that are routinely applied to visual datasets and often are unobservable to a human eye. Therefore, the availability of clean data is no longer guaranteed. One way to alleviate this problem can be to train the networks with noisy data expected to be seen in the real-world. However, the commonly used large scale datasets \cite{15,16} for training the deep learning models do not provide training data with these artifacts and distortions. Therefore, it is imperative to develop training techniques that can give more robust deep learning models while using only available large scale popular datasets that do not cater for these distortions.

The problem discussed in this work is about improving the robustness and stability of deep neural networks. This is a fundamental problem in computer vision and has recently received increased interest by the community \cite{21,26,27}. Our focus is on improving the training process rather than the DNN architecture. We introduce a general training technique that can be applied to any standard state-of-the-art deep learning model and lets them learn a mapping that is more robust and insensitive to input visual perturbations and distortions. We note that a deep neural network can be considered as a mathematical model and the least we can expect from a stable mathematical model is that a small perturbation or distortion in its input will not produce a large change in its behavior. In order to realize this, we utilize some fundamental concepts including Lipschitz functions and Lipschitz continuity. According to the perturbation theory, if the input is perturbed by a small amount, the output of the system stays \textit{close} to its nominal output when there is no perturbation in the input provided that the system dynamics are continuous and locally Lipschitz. In order to motivate the dynamics of the deep neural network to remain locally Lipschitz, we include an additional term in the loss function called $L_{Lipschitz}$.
We provide theoretical justification for the proposed training method in Section \ref{approach}, proving that for admissible distortions in the neighborhood of input image, the Locally Lipschitz neural network is guaranteed to be stable, thus improving the performance in presence of noisy data. We verify the theoretical results by performing extensive experiments on MNIST, CIFAR-10 and STL-10 datasets.

We summarize the paper findings in Table \ref{paperSummary} where resnet-20 network architecture, trained without the proposed method, when presented with distorted input images fails to classify them as the severity of the distortion increases. Even for correctly classified distorted images, the prediction confidence is very low. On the other hand, the same architecture trained with the proposed method when presented with same distorted images correctly classifies them with reasonable prediction confidence.

\section{Related Work}
\label{related_work}
While training the deep neural networks, availability of high quality and artifact-free image data is often assumed. However, this may not always be true due to distortions the images encounter during accusation, transmission and storage phases. Moreover, with the increasing demand of DNN based mobile applications, the assumption for high quality of the availability of high quality input data needs to be relaxed. \cite{18,19} showed that the deep neural networks trained on clean datasets are all susceptible to poor performance when tested to blur and noise distortions while being resilient to compression artifacts such as JPEG and contrast. They propose to train the networks on low quality data to alleviate this problem, which may cause networks to perform poorly to high quality data. The VGG \cite{21} architecture was shown to perform better than AlexNet \cite{22} or GoogleNet \cite{2} to the considered types of distortions.
\cite{2} showed that standard architectures trained on high-quality data suffered significant degradation in performance when tested with distorted data due to blurring or camera motion. They showed that fine-tuning the trained models with a mix of blurry and sharp training examples helps to regain the lost performance to a degree at the cost of minor computational overhead. \cite{25} proposed two approaches to alleviate poor performance due to blurred and noisy images: re-training and fine-tuning with noisy images, showing that fine-tuning is more practical than re-training. \cite{28} also shows that fine-tuned networks on distorted data outperform the original networks when tested on noisy data, but these fine-tuned networks show poor performance on quality distortions that they have not been trained for. \cite{28} propose the concept of mixture of experts ensemble, where various experts are trained on different types of distortions and the final output of the model is the weighted sum of these expert models' outputs. A separate gating network is used to determine these weights.
\cite{26} presents BANG which is the training algorithm that assigns more weight to the correctly classifies samples. Since the correctly classified training samples do not contribute much to the loss as compared to the incorrectly classified training samples, therefore, training is more focused on learning those samples that are badly classified. \cite{26} proved that increasing the contribution of correctly classified training samples in the batch helps flatten the decision space around these training samples, thus training more robust DNNs.
In addition to above mentioned issues, \cite{29} showed the inability of many machine learning models to deal with slightly, but intentionally, perturbed examples which are called adversarial examples. These adversarial examples are indistinguishable to human observers from their original counterparts. Authors in \cite{29} were first to introduce a method of finding adversarial perturbations while \cite{30} introduced a computationally cheaper adversarial example generation algorithm called Fast Gradient Sign Method (FGSM).
Our work differs drastically from \cite{27} as instead of flattening the neural network dynamics function $f$ altogether, we are more focused on setting a soft upper bound on the gradient of $f$ that does not adversely affects the representational power of the neural network. Our work also differs from data augmentation as we propose a way to improve the training process without using any extra training samples, while data augmentation uses standard training techniques and instead increases the number of training samples.

\section{Background}
\label{background}
In this section, we present the basic concepts of Lipschitz functions and Lipschitz continuity.


Let $S$ be an open set in some $\mathbb R^n$. A function $f: \mathbb R^n \rightarrow \mathbb R^m$ is called Lipschitz continuous on $S$ if 
 if there exists a nonnegative constant $L_f$ $\in \mathbb{R}_{\geq 0} $, called a Lipschitz constant of function $f$ on $S$, such that the following condition holds:
\begin{equation}
\| f(x) - f(y)\| \leq L_f \| x-y \| \label{eq:1}
\end{equation}
for all $x, y \in \mathbb S$. 
We call the function $f$ to be locally Lipschitz continuous if for each  $z \in \mathbb{R}^{n}$, there exists a constant $r$ such that $f$ in Lipschitz continuous on the open ball  $B_r(z)$ of center $z$ and radius $r$, where $B_r (z)$ is mathematically written as $B_r (z)= \left\{ y \in \mathbb{R}^{n} : \| y-z \| < r \right\}$.
The function $f$ is said to be globally Lipschitz continuous if it is Lipschitz continuous on its entire domain $\mathbb R^n$.  We note that if the function $f(x$) is Lipschitz continuous with a Lipschitz constant $L_f$, then it is also Lipschitz continuous with any $L$ such that $L \geq Lf$.

Lipschitz continuity is a measure designed to measure the change of the function values versus the change in the independent variable.  Let $f: \mathbb R^n \rightarrow \mathbb R^m$ be a Lipschitz continuous function with a Lipschitz constant $L_f$, so it satisfies  \eqref{eq:1}, i.e.
\begin{equation}\label{Lip2}
\frac{\| f(x) - f(y)\| }{\|x - y\|} \leq L_f
\end{equation}
for all $x \neq y \in \mathbb R^n$. In other words, the average rate of change in the value of $f$ for any pair of points $x$ and $y$ in $\mathbb R^n$ does not exceed the Lipschitz constant $L_f$.   Here we note that the Lipschitz constant $L_f$ depends upon the function $f$. It may vary from being large for one function to being small for another. If $L_f$ is small, then $f(x)$ may only vary a little as the input is changed. But if $L_f$ is large, the function output $f(x)$ may vary a lot with only a small change in its input $x$. In particular, when the Lipschitz function $f$ is real-valued, i.e. $m=1$, then by taking the limit of \eqref{Lip2} as $y \rightarrow x$ we obtain $\| f^\prime (x)\| \leq L_f$, where $f^\prime(x)$ is the derivative function of $f(x)$. In other words, the magnitude of (instantaneous) rate of change in $f$ does not exceed the Lipschitz constant $L_f$ when  the Lipschitz continuous function $f$ is differentiable.

Lipschitz continuity, therefore, quantifies the idea of sensitivity of the function $f(x)$ with respect to its argument using the Lipschitz constant $L_f$.  We note here that the Lipschitz constant $L_f$ represents only the upper bound on how much the function $f(x)$ can change with the change in its input, the actual change might also be smaller than that indicated by $L_f$.

\section{Approach}
\label{approach}
Neural networks can be considered as a sequence of layers that attempt to learn the arbitrary mapping $f \colon X \rightarrow Y$. The network is parameterized with many parameters that are optimized given the training data $x \in X$ and $y \in Y$. Therefore, imposing the condition of Lipschitz continuity on the neural network dynamics implies that a small perturbation in the input will not result in large change at the output of the network, thus increasing the robustness and the stability of the network. Theoretical justification for our approach is provided in the following theorem.

\begin{thm}
Let $\Lambda=\{ y_1, y_2, \dots , y_l\}$ be the set of $ l $ labels used and let $ \rho = 1/2 \min\limits_{1 \leq i < j \leq l} \|y_i-y_j\| $ be half of the minimum distance between any two labels. Let  $ f(x) $ be the neural network dynamics. Let $ L_n $ be the chosen Lipschitz constant hyperparameter. If $ f(x) $ is Lipschitz, then for all distortions $ d $ in input space such that $ \|d\| < \rho/L_n $, $ x $ and $ \tilde{x} $ are guaranteed to be mapped to the same label where $ \tilde{x} $ is the distorted input of the form $ x+d $.
\end{thm}
\begin{proof}
From the Lipschitz assumption, we have $ \|f(x+d)-f(x)\| \leq L_n/\|d\| $. Since we have $ \|d\| < \rho/L_n $, we get $ \|f(x+d)-f(x)\| < \rho $. Since $ f(x) $ is \textit{discrete-valued} in $ \Lambda $, taking into consideration the definition of $ \rho $, we conclude that both $ x $ and $ \tilde{x} $ get mapped to the same label in set $ \Lambda $.
\end{proof}
The Lipschitz property of $ f(x) $ guarantees that for any distortion $ d $ such that $ \|d\| < \rho / L_n $, the output of the distorted input lies within a sphere of radius $ \rho $ about the output of the nominal input where $ \rho $ gives the half of the maximum distance between any two labels. Thus, it is guaranteed that distorted input gets mapped to the same label as the nominal input. For the case when the network is trained without the proposed method, we do not impose any \textit{upper bound} on the slope of $ f(x) $. Therefore, we have Lipschitz constant $ L_n = \infty $ which in Theorem 1 gives $ \|d\|=0 $, which trivially implies that there is no distortion $ d $ for which the network is guaranteed to be robust.

\section{Method}
\label{Method}
Let $ \mathbb{R}^{H x W x C} \rightarrow R^l$, where $l$ denotes the number of labels, represents the mapping performed by the deep neural network. Let $x \in \mathbb{R}^{H x W x C}$ be the input that the network takes, for example an image in the case of a convolutional neural network. In order to encourage the network to be locally Lipschitz continuous, we perturb the network input during the training process with zero mean Gaussian Noise to get a perturbed copy of the input, $\bar{x}$, i.e.
\begin{equation}
\bar{x} = x + \mathit{N}(0,\sigma) \label{eq:4}  
\end{equation}
where we note that in \eqref{eq:4}, $\mathit{N}(0,\sigma)$ has same dimensions as the input image $x$ i.e. $\mathit{N}(0,\sigma) \in \mathbb{R}^{H x W x C}$ and each component of $\mathit{N}(0,\sigma)$ is a single valued zero mean Gaussian random variable $\mathcal{N}(0,\sigma)$ with standard deviation $\sigma$. Here, $\sigma$ is treated as a hyperparameter in the experiments.

In general, the derivative $f'(x)$ of a function $f(x)$ at a point a point $x$  is defined by
\[
f^\prime (x) = \lim_{y\rightarrow x} \frac{f(y) - f(x)}{y - x},
\]
it can be approximated by
\[
f'(x) \approx  \frac {f(y)-f(x)} {y-x}, 
\]
where $y$ is a point \textit{near} $x$. Hence, if we take $y = \bar{x} = x + \mathit{N}(0,\sigma)$
from \eqref{eq:4},  we then have
\begin{equation}
\|f'(x)\| \approx \frac {\|f(\bar{x})-f(x)\|} {\|\bar{x}-x\|} =: k(x). \label{eq:10}  
\end{equation}

In order to encourage the neural network to become locally Lipschitz continuous, we add an additional term, called $L_{Lipschitz}$, in the usual loss function, termed here as $L_{usual}$, to get an aggregated loss function $L$, i.e.
\[
L = L_{usual} + L_{Lipschitz}, 
\]
where $L_{usual}$ is the loss term corresponding to the task to be performed by the network, for example cross-entropy loss, while $L_{Lipschitz}$ is defined as:
\begin{equation}
L_{Lipschitz} = \beta * \max( 0 ,  k(x) - L_n) \label{eq:12}
\end{equation}
where $\beta$ is the weighting factor for the added loss term $L_{Lipschitz}$, $L_n$ serves the purpose of the Lipschitz constant for the neural network dynamics, and $k(x)$ is given in \eqref{eq:10}. We treat both $\beta$ and $L_n$ as hyperparameters.

The effect of the hyperparameters will be studied in Section \ref{Secsitivity_Analysis}.

\section{Experiments}
\label{experiments}
In order to evaluate our approach, we tested our proposed training procedure with MNIST \cite{20}, CIFAR-10 \cite{23} and STL-10 \cite{24} datasets. Details about these experiments and their results are explained in following subsections. When we train the network without using the proposed training method, we refer to the training method as \textit{standard training method}.

\paragraph{Justification for using Gaussian Noise:} In experiments, we use Gaussian noise to corrupt test data. To see why Gaussian model can approximate realistic distortions, we see that any distortion of an image $x$ can always be expressed as $\bar x = T_\sigma(x)$, where $T_\sigma(\cdot)$ is a map {\it close to the identity map}, i.e. $T_0(x) = x$,  parameterized by a parameter $\sigma$. Hence, for all small values of $\sigma$, $\bar x = T_\sigma(x) = T_0(x)  + O(|\sigma|) = x + O(|\sigma|)$ in Taylor expansion of $T_\sigma(x)$ in $\sigma$ around $\sigma = 0$, where $O(|\sigma|)$ represents the terms of order 1 or higher in $\sigma$ and can be interpreted as a perturbation term that vanishes when $\sigma=0$.  Hence, it is reasonable to use  Gaussian noises $\mathit{N}(0,\sigma)$ to simulate various realistic distortions to the image $x$.

Due to space constraints, some tables and figures are given in the supplementary material and will be referenced in the subsequent sections as required.
\begin{table}[t]
	\caption{Classification accuracies for experiments with MNIST. Results are shown for various levels of distortions in test dataset as described by the value of $\sigma_{test}$. Here $\sigma_{test}=0.0$ corresponds to undistorted test data. We used $\beta=10$ for MNIST experiments.}	
	\label{mnist-results-table}
	\center
	\begin{tabular}{l c c c}
	\toprule
	\multirow{2}{*}{Network Training Details} & \multicolumn{3}{c}{$\sigma_{test}$} \\ \cmidrule(l){2-4} &  $0.0$ & $0.5$ & $1.0$ \\
	\midrule
	standard method &	0.97 &	0.92 &	0.65	\\
	$\sigma_{train}=0.5$, $L_n=0.01$ &	0.98 &	0.95 &	0.70  \\
	$\sigma_{train}=0.75$, $L_n=0.005$ &	0.98 &	0.96 &	0.78  \\
	$\sigma_{train}=0.75$, $L_n=0.01$ &	0.98 &	0.96 &	0.77  \\
	\bottomrule
	\end{tabular}
\end{table}

\subsection{MNIST}
\label{mnist}
\subsubsection{Experiment Details}
\label{mnist-experiment-details}
We used a convolutional neural network consisting of one convolutional layer, one fully-connected layer and an output layer for experiments with MNIST dataset. 5 epochs of 550 iteration were performed and learning rate was set to $10^{-4}$. For training the network with standard training method, we set $\beta=0$ in \eqref{eq:12}. Network was trained with and without the proposed training mechanism. $\left(\sigma_{train},\beta,L_{n}\right) \in \{0.5,0.75\} \times \{10\} \times \{0.005,0.01\}$ were used as hyperparameters.

\begin{figure*}
\centering
\subfloat[Model trained with $\sigma_{train}=0.5, L_n=0.01$ and $\beta=10$]{\includegraphics[scale=0.5]{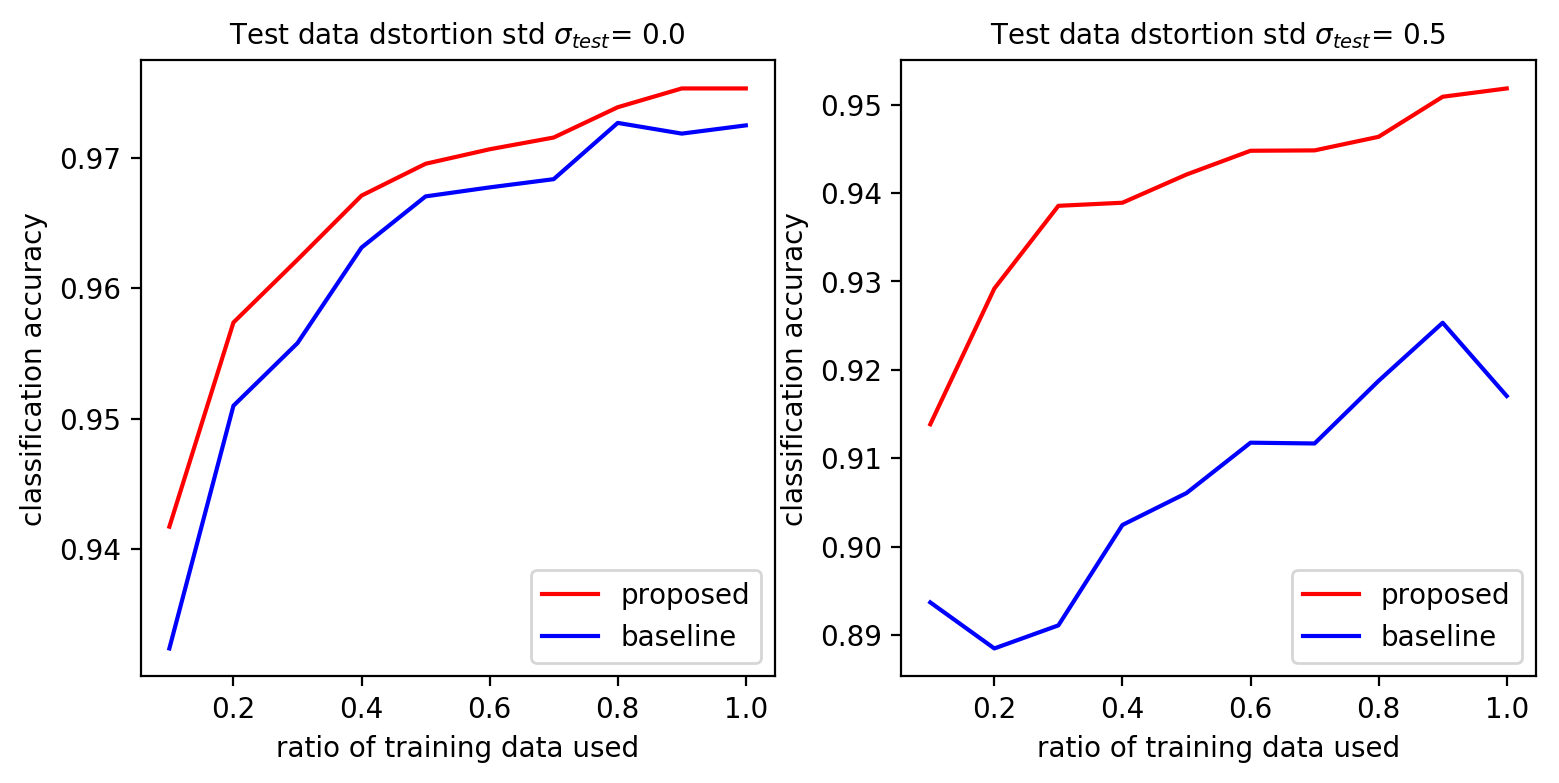}}

\subfloat[Model trained with $\sigma_{train}=0.75, L_n=0.005$ and $\beta=10$]{\includegraphics[scale=0.5]{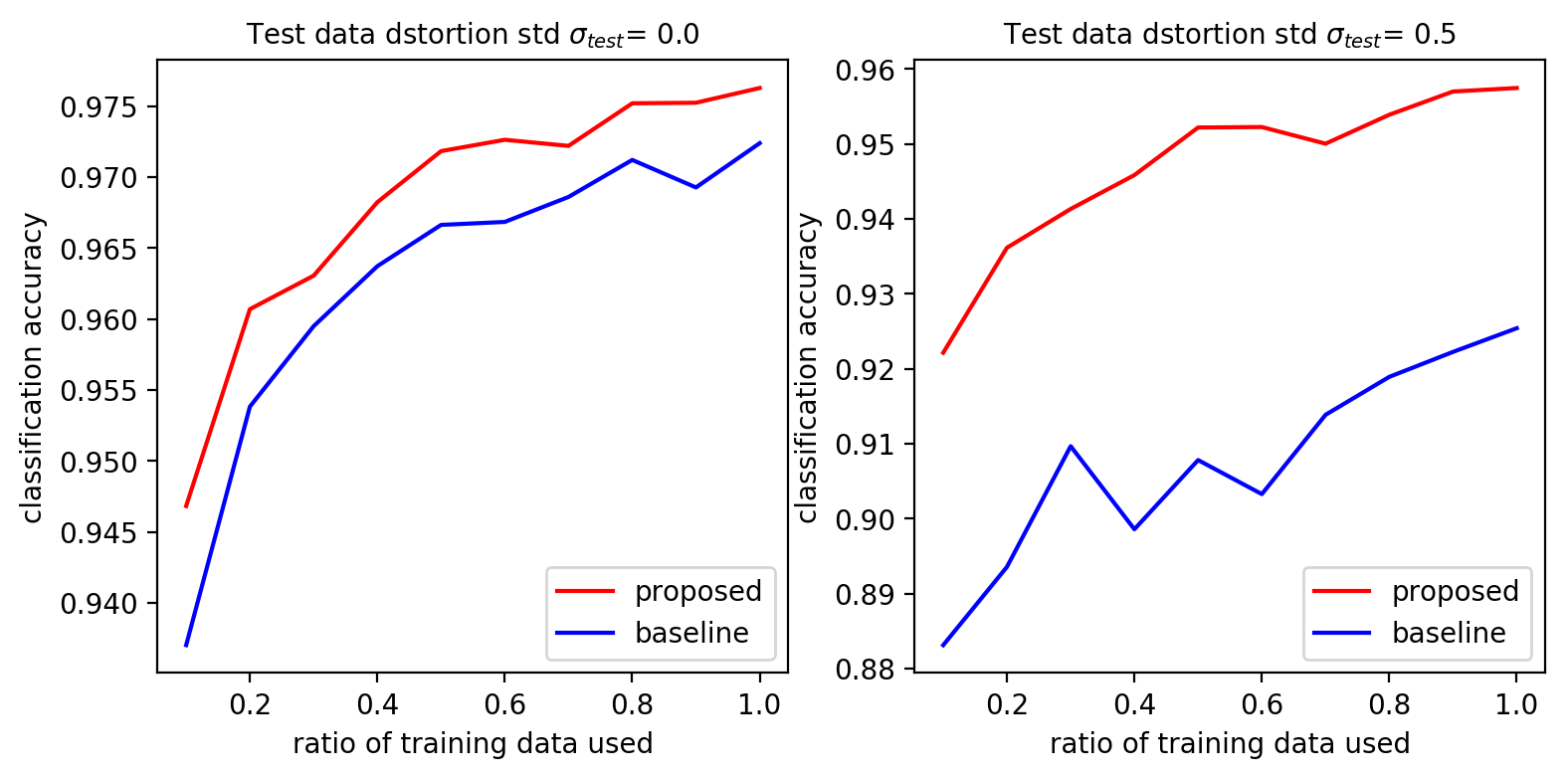}}

\subfloat[Model trained with $\sigma_{train}=0.75, L_n=0.01$ and $\beta=10$]{\includegraphics[scale=0.5]{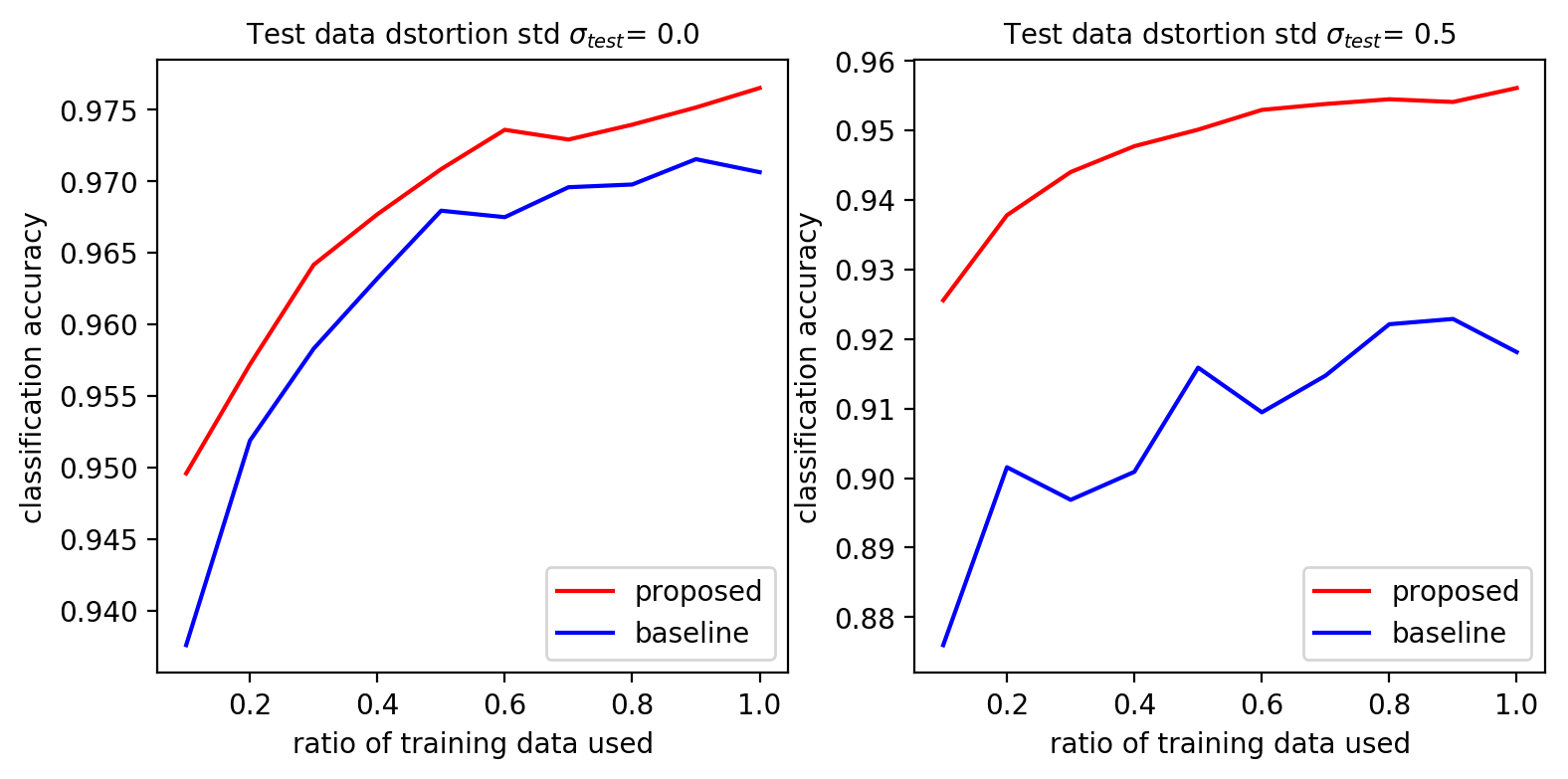}}

\caption{Plots of classification accuracy versus ratio of training data used in training process.}
\label{mnist-figure-ratios}
\end{figure*}

We tested trained networks with test data distorted with zero mean Gaussian noise with standard deviation values of $\sigma_{test}=0.0,0.5$ and $1.0$. Networks trained with various percentages of training data were also tested.

We also investigated the effects of using only a proportion of training data for training purpose. We trained the networks with various percentages of training data and tested them on entire test data. We randomly sample a percentage of training data at the start of training. We hypothesize that a robust neural network trained with only a portion of training data should be able to generalize well across the entire test dataset.

\begin{table*}[bt]
	\caption{Top-1 accuracies for models trained with $\sigma_{train}=0.25$ on CIFAR-10 dataset}
	\label{cifar10-table-results-025}
	\center
	\begin{tabular}{p{2cm} c c c c c c c c c c}
	\toprule
	\multicolumn{11}{c}{For ResNet-20 Architecture}	\\
	\midrule
	\multirow{3}{*}{Training}	& \multirow{3}{*}{$L_n$} & \multicolumn{9}{c}{Test Data Distortion $\sigma_{test}$} \\
	\cmidrule(l){3-11} & & \multicolumn{3}{c}{$\sigma_{test}=0.0$} & \multicolumn{3}{c}{$\sigma_{test}=0.3$} & \multicolumn{3}{c}{$\sigma_{test}=0.5$}	\\
	\cmidrule(l){3-5} \cmidrule(l){6-8} \cmidrule(l){9-11} & & $\beta=1$ & $\beta=5$ & $\beta=10$ & $\beta=1$ & $\beta=5$ & $\beta=10$ & $\beta=1$ & $\beta=5$ & $\beta=10$ \\
	\cmidrule(l){3-3} \cmidrule(l){4-4} \cmidrule(l){5-5} \cmidrule(l){6-6} \cmidrule(l){7-7} \cmidrule(l){8-8} \cmidrule(l){9-9} \cmidrule(l){10-10} \cmidrule(l){11-11} 
	standard & - & \textbf{92.77} & \textbf{92.77} & \textbf{92.77} & 38.28 & 38.28 & 38.28 & 18.01 & 18.01 & 18.01 \\
	\multirow{2}{*}{proposed} & 0.01 & 88.02	& 82.49	& 88.18	& 50.56	& 45.80	& \textbf{58.56}	& 23.49	& 21.68	& \textbf{29.70} \\ & 0.1 & 88.86	& 88.21	& 89.00	& \textbf{63.01}	& \textbf{59.05}	& 57.92	& \textbf{34.73}	& \textbf{31.17}	& 25.84 \\
	\midrule
	
	\multicolumn{11}{c}{For Preresnet-20 Architecture}	\\
	\midrule
	\multirow{3}{*}{Training}	& \multirow{3}{*}{$L_n$} & \multicolumn{9}{c}{Test Data Distortion $\sigma_{test}$} \\
	\cmidrule(l){3-11} & & \multicolumn{3}{c}{$\sigma_{test}=0.0$} & \multicolumn{3}{c}{$\sigma_{test}=0.3$} & \multicolumn{3}{c}{$\sigma_{test}=0.5$}	\\
	\cmidrule(l){3-5} \cmidrule(l){6-8} \cmidrule(l){9-11} & & $\beta=1$ & $\beta=5$ & $\beta=10$ & $\beta=1$ & $\beta=5$ & $\beta=10$ & $\beta=1$ & $\beta=5$ & $\beta=10$ \\
	\cmidrule(l){3-3} \cmidrule(l){4-4} \cmidrule(l){5-5} \cmidrule(l){6-6} \cmidrule(l){7-7} \cmidrule(l){8-8} \cmidrule(l){9-9} \cmidrule(l){10-10} \cmidrule(l){11-11} 
	standard & - & \textbf{92.59} & \textbf{92.59} & \textbf{92.59} & 30.80 & 30.80 & 30.80 & 15.91 & 15.91 & 15.91 \\
	\multirow{2}{*}{proposed} & 0.01 & 86.91	& 87.86	& 88.34	& \textbf{63.52}	& 58.86	& 59.47	& 32.84	& 26.99	& 26.33 \\ & 0.1 & 86.55	& 88.11	& 87.80	& 58.51	& \textbf{61.25}	& \textbf{64.98}	& \textbf{35.14}	& \textbf{32.76}	& \textbf{39.41} \\
	\bottomrule
	
	\end{tabular}
\end{table*}

\begin{table}[bt]
	\caption{Top-1 accuracies for models trained with $\sigma_{train}=0.05$ on CIFAR-10 dataset}
	\label{cifar10-table-results-005}
	\centering
	\begin{tabular}{p{2cm} c c c c c c c c c c}
	\toprule
	\multicolumn{11}{c}{For ResNet-20 Architecture}	\\
	\midrule
	\multirow{3}{*}{Training}	& \multirow{3}{*}{$L_n$} & \multicolumn{9}{c}{Test Data Distortion $\sigma_{test}$} \\
	\cmidrule(l){3-11} & & \multicolumn{3}{c}{$\sigma_{test}=0.0$} & \multicolumn{3}{c}{$\sigma_{test}=0.3$} & \multicolumn{3}{c}{$\sigma_{test}=0.5$}	\\
	\cmidrule(l){3-5} \cmidrule(l){6-8} \cmidrule(l){9-11} & & $\beta=1$ & $\beta=5$ & $\beta=10$ & $\beta=1$ & $\beta=5$ & $\beta=10$ & $\beta=1$ & $\beta=5$ & $\beta=10$ \\
	\cmidrule(l){3-3} \cmidrule(l){4-4} \cmidrule(l){5-5} \cmidrule(l){6-6} \cmidrule(l){7-7} \cmidrule(l){8-8} \cmidrule(l){9-9} \cmidrule(l){10-10} \cmidrule(l){11-11} 
	standard & - & \textbf{92.77} & \textbf{92.77} & \textbf{92.77} & \textbf{38.28} & 38.28 & 38.28 & \textbf{18.01} & 18.01 & 18.01 \\
	\multirow{2}{*}{proposed} & 0.01 & 92.39 & 92.44	& 92.68	& 32.5 & 36.56	& 34.09	& 15.68	& 19.65	& 16.33 \\ & 0.1 & 92.57	& 92.32	& \textbf{93.02}	& 34.15	& \textbf{38.48}	& \textbf{42.06}	& 13.67	& \textbf{20.63}	& \textbf{21.56} \\
	\midrule
	
	\multicolumn{11}{c}{For Preresnet-20 Architecture}	\\
	\midrule
	\multirow{3}{*}{Training}	& \multirow{3}{*}{$L_n$} & \multicolumn{9}{c}{Test Data Distortion $\sigma_{test}$} \\
	\cmidrule(l){3-11} & & \multicolumn{3}{c}{$\sigma_{test}=0.0$} & \multicolumn{3}{c}{$\sigma_{test}=0.3$} & \multicolumn{3}{c}{$\sigma_{test}=0.5$}	\\
	\cmidrule(l){3-5} \cmidrule(l){6-8} \cmidrule(l){9-11} & & $\beta=1$ & $\beta=5$ & $\beta=10$ & $\beta=1$ & $\beta=5$ & $\beta=10$ & $\beta=1$ & $\beta=5$ & $\beta=10$ \\
	\cmidrule(l){3-3} \cmidrule(l){4-4} \cmidrule(l){5-5} \cmidrule(l){6-6} \cmidrule(l){7-7} \cmidrule(l){8-8} \cmidrule(l){9-9} \cmidrule(l){10-10} \cmidrule(l){11-11} 
	standard & - & \textbf{92.59} & 92.59 & \textbf{92.59} & 30.80 & 30.80 & 30.80 & 15.91 & 15.91 & 15.91 \\
	\multirow{2}{*}{proposed} & 0.01 & 92.50	& 92.40	& 92.56	& \textbf{41.31}	& \textbf{35.18}	& \textbf{34.66}	& \textbf{22.25}	& \textbf{18.39}	& \textbf{16.46} \\ & 0.1 & 92.36	& \textbf{92.66}	& 92.34	& 34.92	& 28.37	& 33.67	& 17.54	& 16.12	& 14.83 \\
	\bottomrule
	
	\end{tabular}
\end{table}

\begin{table}[bt]
	\caption{Top-1 accuracies for models trained with $\sigma_{train}=0.5$ on CIFAR-10 dataset}
	\label{cifar10-table-results-05}
	\center
	\begin{tabular}{p{2cm} c c c c c c c c c c}
	\toprule
	\multicolumn{11}{c}{For ResNet-20 Architecture}	\\
	\midrule
	\multirow{3}{*}{Training}	& \multirow{3}{*}{$L_n$} & \multicolumn{9}{c}{Test Data Distortion $\sigma_{test}$} \\
	\cmidrule(l){3-11} & & \multicolumn{3}{c}{$\sigma_{test}=0.0$} & \multicolumn{3}{c}{$\sigma_{test}=0.3$} & \multicolumn{3}{c}{$\sigma_{test}=0.5$}	\\
	\cmidrule(l){3-5} \cmidrule(l){6-8} \cmidrule(l){9-11} & & $\beta=1$ & $\beta=5$ & $\beta=10$ & $\beta=1$ & $\beta=5$ & $\beta=10$ & $\beta=1$ & $\beta=5$ & $\beta=10$ \\
	\cmidrule(l){3-3} \cmidrule(l){4-4} \cmidrule(l){5-5} \cmidrule(l){6-6} \cmidrule(l){7-7} \cmidrule(l){8-8} \cmidrule(l){9-9} \cmidrule(l){10-10} \cmidrule(l){11-11} 
	standard & - & \textbf{92.77} & \textbf{92.77} & \textbf{92.77} & 38.28 & 38.28 & 38.28 & 18.01 & 18.01 & 18.01 \\
	\multirow{2}{*}{proposed} & 0.01 & 82.97	& 82.72	& 82.55	& \textbf{70.31}	& \textbf{67.60}	& 66.89	& \textbf{46.47}	& \textbf{43.91}	& 41.15 \\ & 0.1 & 81.36	& 83.20	& 84.62	& 58.55	& 60.09	& \textbf{72.40}	& 36.25	& 40.75	& \textbf{45.53} \\
	\midrule
	
	\multicolumn{11}{c}{For Preresnet-20 Architecture}	\\
	\midrule
	\multirow{3}{*}{Training}	& \multirow{3}{*}{$L_n$} & \multicolumn{9}{c}{Test Data Distortion $\sigma_{test}$} \\
	\cmidrule(l){3-11} & & \multicolumn{3}{c}{$\sigma_{test}=0.0$} & \multicolumn{3}{c}{$\sigma_{test}=0.3$} & \multicolumn{3}{c}{$\sigma_{test}=0.5$}	\\
	\cmidrule(l){3-5} \cmidrule(l){6-8} \cmidrule(l){9-11} & & $\beta=1$ & $\beta=5$ & $\beta=10$ & $\beta=1$ & $\beta=5$ & $\beta=10$ & $\beta=1$ & $\beta=5$ & $\beta=10$ \\
	\cmidrule(l){3-3} \cmidrule(l){4-4} \cmidrule(l){5-5} \cmidrule(l){6-6} \cmidrule(l){7-7} \cmidrule(l){8-8} \cmidrule(l){9-9} \cmidrule(l){10-10} \cmidrule(l){11-11} 
	standard & - & \textbf{92.59} & \textbf{92.59} & \textbf{92.59} & 30.80 & 30.80 & 30.80 & 15.91 & 15.91 & 15.91 \\
	\multirow{2}{*}{proposed} & 0.01 & 82.43	& 80.88	& 80.00	& \textbf{70.83} & \textbf{64.34}	& \textbf{72.08}	& \textbf{45.54}	& \textbf{33.10}	& 47.19 \\ & 0.1 & 80.08	& 85.17	& 82.42	& 53.92	& 58.37	& 65.94	& 31.77	& 26.48	& \textbf{47.28} \\
	\bottomrule
	
	\end{tabular}
\end{table}

\subsubsection{Results}
\label{results}
Table \ref{mnist-results-table} presents classification accuracies for models trained with different combinations of hyperparameters. We see that networks trained with Lipschitz continuity loss perform better than the network obtained with standard training procedure. With undistorted test data, the gain in performance is small but as the severity of distortion increases, the networks trained with proposed method show significant performance improvement over network trained with standard training process. As the value of $ L_n $ is increased keeping other hyperparameters the same, the performance slightly deteriorates in accordance with the conclusion of Theorem 1, where the region of admissible distortions $ d $ decreases as $ L_n $ is increased i.e. $\|d\| \leq \rho / L_n$.

In order to test the robustness of proposed training procedure, we trained the networks with various portions of training data. These models were then tested with entire test dataset, undistorted as well as distorted ($\sigma_{test}=0.5$) . Figure \ref{mnist-figure-ratios} shows that networks trained with Lipschitz loss always perform better than those trained with standard training process, thus proving their robustness.
\subsection{CIFAR-10}
\label{cifar}
\subsubsection{Experiment Details}
\label{cifar-experiment-details}
We used ResNet-20 [4] and PreResNet-20 [17] as our network architectures for classification task with CIFAR-10 dataset. Both networks have 16-16-32-64 channels and 0.26 million parameters each. Each model was trained for 300 epochs with batch size of 128 and learning rate of 0.1. Learning rate was decreased by a factor of 10 first at epoch 150 and then at epoch 225. $\left(\sigma_{train},\beta,L_{n}\right) \in \{0.05,0.25,0.5\} \times \{1,5,10\} \times \{0.001,0.1\}$ were used as hyperparameters. For training the network with standard training method, we set $\beta=0$ in \eqref{eq:12}.

We tested the trained networks with corrupted test data generated by distorting the test data set with zero mean Gaussian Noise having standard deviation values ranging from $\sigma_{test}=0.0$  to $\sigma_{test}=0.5$ with step size of $0.01$.
\subsubsection{Results}
\label{cifar-results}
Table \ref{cifar10-table-results-025} shows the top-$1$ classification accuracies for networks trained with $\sigma_{train}=0.25$ and $\sigma_{test}$ values of $0.0, 0.3$ and $0.5$. Similarly, Table \ref{cifar10-table-results-005} and Table \ref{cifar10-table-results-05} show results in similar fashion for $\sigma_{train}=0.05$ and $\sigma_{train}=0.5$  respectively. Figure \ref{cifar_figure_005} , \ref{cifar_figure_025} and \ref{cifar_figure_05} show plots for test accuracies versus $\sigma_{test} = 0.0-0.5$ for networks trained with $\sigma_{train}=0.05, 0.25, 0.5$ respectively for better visualization.

We see that the models trained with $\sigma_{train}=0.05$ perform comparable to the original baseline with the undistorted test data. As the distortion severity is increased, they perform better than the baseline confirming that they are more robust to input visual distortions. As the value of $\sigma_{train}$ is increased, we get the models that tend to lose performance with the undistorted dataset but perform much better as the distortion severity in increased. Therefore, models trained with increased values of $\sigma_{train}$ are much more robust and insensitive to input distortions with some loss in performance with undistorted input data. We also note that as the value of $\beta$ is increased, the performance difference of models trained with different $L_n$ values tends to diminish as they start to performance equally well. This is due to high value of $\beta$ that makes the effect of different $L_n$ values in the training loss ineffective.

\begin{table*}[bt]
	\caption{Top-1 accuracies for models trained with $\sigma_{train}=0.25$ on STL-10 dataset}
	\label{stl10-table-results-025}
	\center
	\begin{tabular}{p{2cm} c c c c c c c c c c}
	\toprule
	\multirow{3}[5]{*}{Training}	& \multirow{3}[5]{*}{$L_n$} & \multicolumn{9}{c}{Test Data Distortion $\sigma_{test}$} \\
	\cmidrule(l){3-11} & & \multicolumn{3}{c}{$\sigma_{test}=0.0$} & \multicolumn{3}{c}{$\sigma_{test}=0.3$} & \multicolumn{3}{c}{$\sigma_{test}=0.5$}	\\
	\cmidrule(l){3-5} \cmidrule(l){6-8} \cmidrule(l){9-11} & & $\beta=1$ & $\beta=5$ & $\beta=10$ & $\beta=1$ & $\beta=5$ & $\beta=10$ & $\beta=1$ & $\beta=5$ & $\beta=10$ \\
	\cmidrule(l){1-1} \cmidrule(l){2-2} \cmidrule(l){3-3} \cmidrule(l){4-4} \cmidrule(l){5-5} \cmidrule(l){6-6} \cmidrule(l){7-7} \cmidrule(l){8-8} \cmidrule(l){9-9} \cmidrule(l){10-10} \cmidrule(l){11-11} 
	standard & - & \textbf{80.44}	& \textbf{80.44}	& \textbf{80.44	} & 50.67	& 50.67	& 50.67	& 34.94	& 34.94	& 34.94 \\
	\multirow{2}{*}{proposed} & 0.01 & 75.65	& 77.41	& 78.19	& 66.11	& 59.00	& 62.91	& \textbf{48.77}	& 40.94	& \textbf{44.52} \\ & 0.1 & 78.88	& 79.71	& 77.21	& \textbf{68.34}	& \textbf{65.47}	& \textbf{64.60}	& 47.42	& \textbf{46.26}	& 42.24 \\
	\bottomrule
	\end{tabular}
\end{table*}

\begin{table*}[bt]
	\caption{Top-1 accuracies for models trained with $\sigma_{train}=0.05$ on STL-10 dataset}
	\label{stl10-table-results-005}
	\center
	\begin{tabular}{p{2cm} c c c c c c c c c c}
	\toprule
	\multirow{3}[5]{*}{Training}	& \multirow{3}[5]{*}{$L_n$} & \multicolumn{9}{c}{Test Data Distortion $\sigma_{test}$} \\
	\cmidrule(l){3-11} & & \multicolumn{3}{c}{$\sigma_{test}=0.0$} & \multicolumn{3}{c}{$\sigma_{test}=0.3$} & \multicolumn{3}{c}{$\sigma_{test}=0.5$}	\\
	\cmidrule(l){3-5} \cmidrule(l){6-8} \cmidrule(l){9-11} & & $\beta=1$ & $\beta=5$ & $\beta=10$ & $\beta=1$ & $\beta=5$ & $\beta=10$ & $\beta=1$ & $\beta=5$ & $\beta=10$ \\
	\cmidrule(l){1-1} \cmidrule(l){2-2} \cmidrule(l){3-3} \cmidrule(l){4-4} \cmidrule(l){5-5} \cmidrule(l){6-6} \cmidrule(l){7-7} \cmidrule(l){8-8} \cmidrule(l){9-9} \cmidrule(l){10-10} \cmidrule(l){11-11} 
	standard & - & 80.44	& 80.44	& 80.44	& 50.67	& 50.67	& 50.67	& 34.94	& 34.94	& 34.94 \\
	\multirow{2}{*}{proposed} & 0.01 & 80.44	& 80.88	& \textbf{80.79}	& \textbf{56.54}	& \textbf{67.81}	& \textbf{56.56}	& \textbf{41.74}	& \textbf{51.58}	& \textbf{37.25} \\ & 0.1 & \textbf{81.34}	& \textbf{80.90}	& 80.64	& 55.46	& 52.88	& 49.86	& 39.51	& 39.51	& 34.84 \\

	\bottomrule
	
	\end{tabular}
\end{table*}

\begin{table*}[t]
	\caption{Top-1 accuracies for models trained with $\sigma_{train}=0.5$ on STL-10 dataset}
	\label{stl10-table-results-05}
	\center
	\begin{tabular}{p{2cm} c c c c c c c c c c}
	\toprule
	\multirow{3}[5]{*}{Training}	& \multirow{3}[5]{*}{$L_n$} & \multicolumn{9}{c}{Test Data Distortion $\sigma_{test}$} \\
	\cmidrule(l){3-11} & & \multicolumn{3}{c}{$\sigma_{test}=0.0$} & \multicolumn{3}{c}{$\sigma_{test}=0.3$} & \multicolumn{3}{c}{$\sigma_{test}=0.5$}	\\
	\cmidrule(l){3-5} \cmidrule(l){6-8} \cmidrule(l){9-11} & & $\beta=1$ & $\beta=5$ & $\beta=10$ & $\beta=1$ & $\beta=5$ & $\beta=10$ & $\beta=1$ & $\beta=5$ & $\beta=10$ \\
	\cmidrule(l){1-1} \cmidrule(l){2-2} \cmidrule(l){3-3} \cmidrule(l){4-4} \cmidrule(l){5-5} \cmidrule(l){6-6} \cmidrule(l){7-7} \cmidrule(l){8-8} \cmidrule(l){9-9} \cmidrule(l){10-10} \cmidrule(l){11-11} 
	standard & - & \textbf{80.44} & \textbf{80.44}	& \textbf{80.44}	& 50.67	& 50.67	& 50.67	& 34.94	& 34.94	& 34.94 \\
	\multirow{2}{*}{proposed} & 0.01 & 73.16	& 72.08	& 70.90	& 60.19	& 66.97	& 65.89	& \textbf{48.83}	& 53.21	& 51.33 \\ & 0.1 & 71.75	& 75.35	& 71.83	& \textbf{62.79}	& \textbf{69.10}	& \textbf{67.88}	& 46.40	& \textbf{55.95}	& \textbf{52.21} \\

	\bottomrule
	
	\end{tabular}
\end{table*}

\subsection{STL-10}
\subsubsection{Experiment Details}
We used PreResNet-32 \cite{2} as our baseline architecture for classification task with STL-10 dataset. The network has 16-16-32-64 channels and 0.46 million parameters. Training conditions and hyperparameters' values are same as for CIFAR-10 experiments. Test data was also generated similar to CIFAR-10 experiments.

\subsubsection{Results}
Table \ref{stl10-table-results-025} shows the top-$1$ classification accuracies for networks trained with $\sigma_{train}=0.25$ and $\sigma_{test}$ values of $0.0$, $0.3$ and $0.5$. Similarly, Table \ref{stl10-table-results-005} and Table \ref{stl10-table-results-05} show results in similar fashion for networks trained with $\sigma_{train}=0.05$ and $\sigma_{train}=0.5$  respectively. Figure \ref{stl_figure_005} , \ref{stl_figure_025} and \ref{stl_figure_05} in the supplementary material show plots for test accuracies versus $\sigma_{test} = 0.0-0.5$ for networks trained with $\sigma_{train}=0.05, 0.25, 0.5$ respectively for better visualization.

We see that the models trained with $\sigma_{train}=0.05$ perform comparable to the original baseline with the undistorted test data. As the distortion severity is increased, they perform better than the baseline confirming that they are more robust to input visual distortions. As the value of $\sigma_{train}$ is increased, we get the models that tend to lose performance with the undistorted dataset but perform much better as the distortion severity in increased. Therefore, models trained with increased values of $\sigma_{train}$ are much more robust and insensitive to input distortions with some loss in performance with undistorted input data. We also note that as the value of $\beta$ is increased, the performance difference of models trained with different $L_n$ values tends to diminish as they start to performance equally well. This is due to high value of $\beta$ that makes the effect of different $L_n$ values in the training loss ineffective.

\section{Sensitivity Analysis of Hyperparameters}
\label{Secsitivity_Analysis}
The impact of hyperparameters is best studied using the sensitivity analysis. The hyperparameters introduced in this study are $\left(\sigma_{train},\beta,L_{n}\right) \in \{0.05,0.25,0.5\} \times \{1,5,10\} \times \{0.01,0.1\}$. For sensitivity analysis, let's take nominal values of hyperparameters be $\left(\sigma_{train},\beta,L_{n}\right) = \left(0.25,5,0.01\right) $. Let \textit{acc} denote the percentage accuracy of the model trained with Lipschitz term in loss function. We change the hyperparameters $ \sigma_{train},\beta $ and $ L_{n} $ as follows: $ \Delta\sigma_{train} = 0.25, \Delta \beta = 5 $ and $ \Delta L_{n} = 0.09 $. Experiments are performed with new hyperparameters values on CIFAR-10 dataset. The sensitivities of model performance with respect to $ \sigma_{train},\beta $ and $ L_{n} $ are given as: 
\[
\Delta acc/\Delta\sigma_{train} = 87.20, \quad \Delta acc/\Delta\beta = 2.55
\]
and
\[
  \Delta acc/\Delta L_n = -28.89
\]  
respectively. 

We see that the network performance is most sensitive to change in $\sigma_{train}$. Performance is least sensitive to change in $\beta$.  Performance is fairly sensitive to change in $L_n$ where the negative value of $\Delta acc/\Delta L_n$ indicates that the performance deteriorates as $L_n$ increases, which is consistent with the conclusion of Theorem 1 in Section \ref{approach} where the radius of admissible distortions $d$ is inversely proportional to the  magnitude of $L_n$ i.e. $\|d\| \leq \rho / L_n$.

\section{Conclusion}
In this paper, we presented a method for training neural networks using Lipschitz continuity that can be used to make them more robust to input visual perturbations. We provide theoretical justification and experimental demonstration about the effectiveness of our method using existing neural network architectures in the presence of input perturbations. Our approach is, therefore, easy-to-use and effective as it improves the network robustness and stability without using data augmentation or additional training data.

\section{Acknowledgement}
This research has been in part supported by the ICT R\&D program of MSIP/IITP [2016-0-00563, Research on Adaptive Machine Learning Technology Development for Intelligent Autonomous Digital Companion].


\begin{figure}[ht]
\centering
\subfloat[Resnet-20 trained with $\sigma_{train}=0.05$]{\includegraphics[scale=0.4]{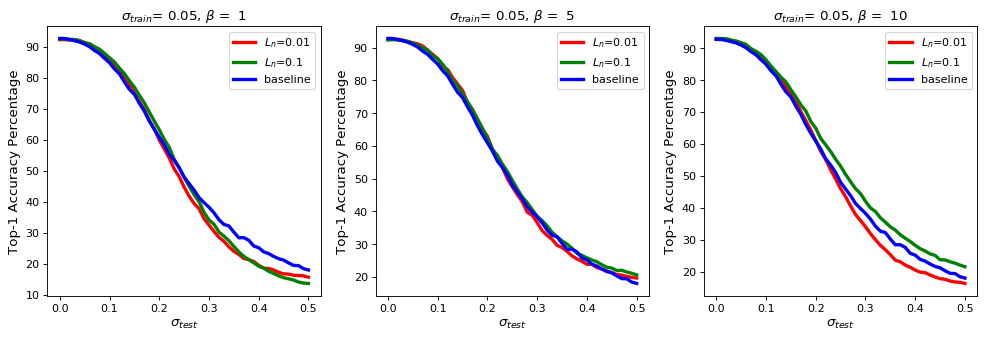}}

\subfloat[Preresnet-20 trained with $\sigma_{train}=0.05$]{\includegraphics[scale=0.4]{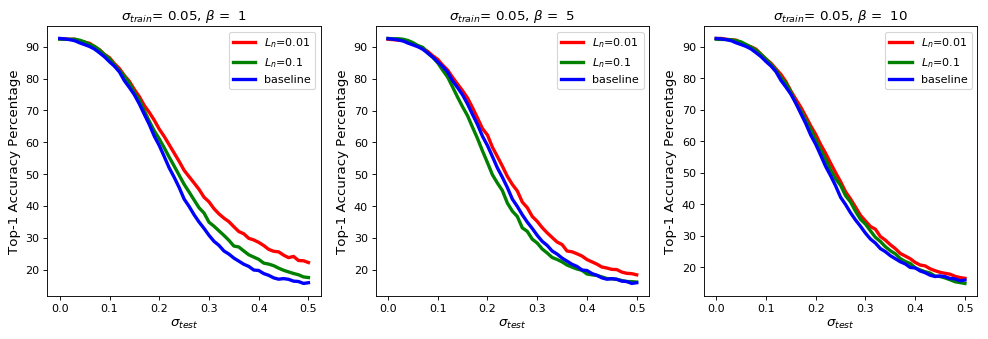}}
\caption{Plots of the top-1 CIFAR-10 test accuracies for models trained with $\sigma_{train}=0.05$ and with standard training. (a) shows results for resnet-20 and (b) shows results for preresnet-20.}
\label{cifar_figure_005}
\end{figure}
\begin{figure}[ht]
\centering
\includegraphics[width=5in]{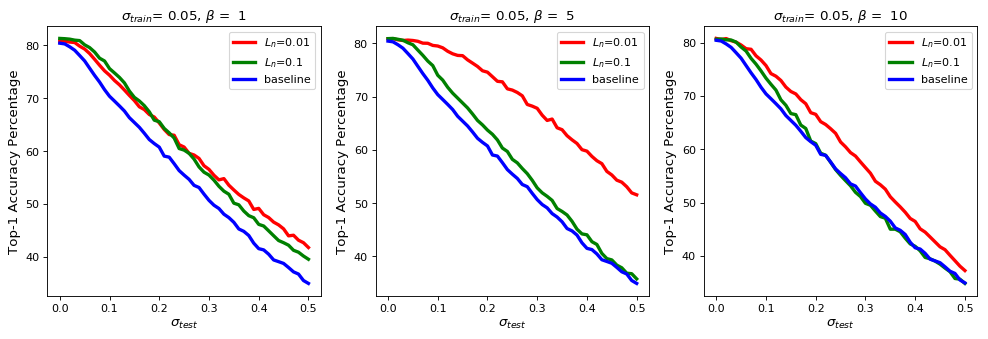}
\caption{Plots of top-1 STL-10 test accuracies for $\sigma_{train}=0.05$}
\label{stl_figure_005}
\end{figure}

\begin{figure}[ht]
\centering
\subfloat[Resnet-20 trained with $\sigma_{train}=0.25$]{\includegraphics[scale=0.4]{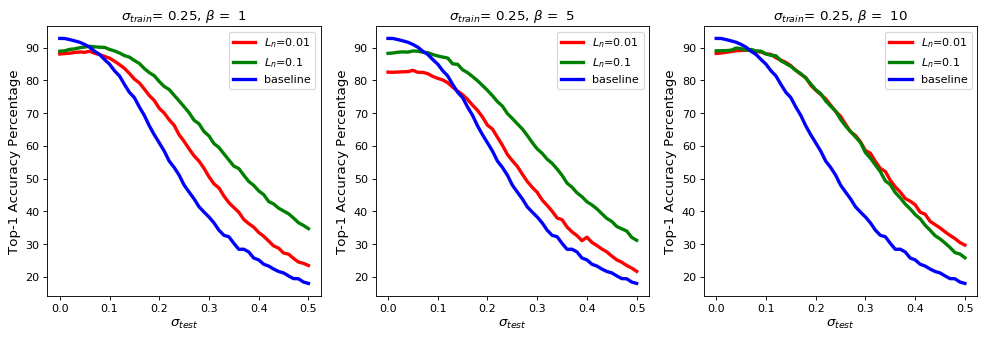}}

\subfloat[Preresnet-20 trained with $\sigma_{train}=0.25$]{\includegraphics[scale=0.4]{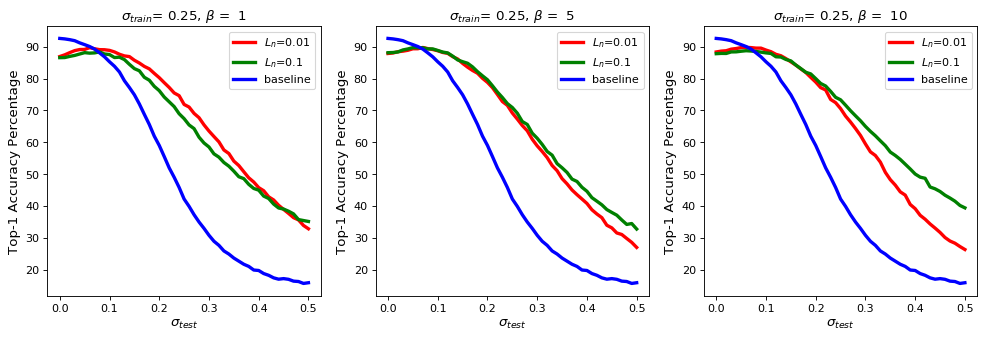}}
\caption{Plots of the top-1 CIFAR-10 test accuracies for $\sigma_{train}=0.25$}
\label{cifar_figure_025}
\end{figure}

\begin{figure}[ht]
\centering
\includegraphics[width=5in]{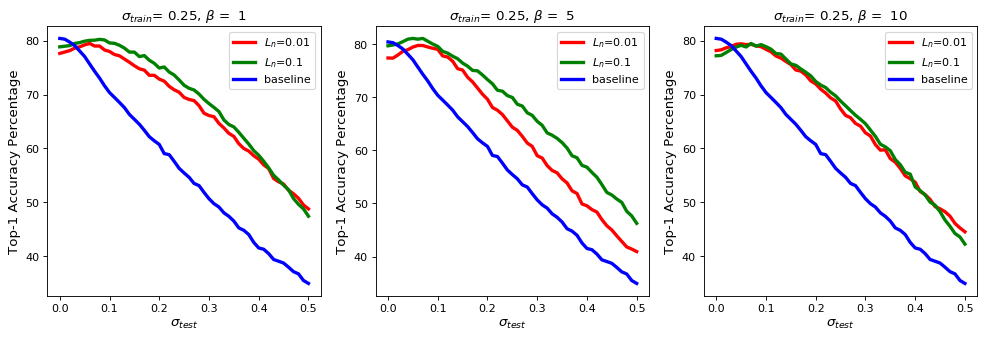}
\caption{Plots of the top-1 STL-10 test accuracies for $\sigma_{train}=0.25$}
\label{stl_figure_025}
\end{figure}

\begin{figure}[ht]
\centering
\subfloat[Resnet-20 trained with $\sigma_{train}=0.5$]{\includegraphics[scale=0.4]{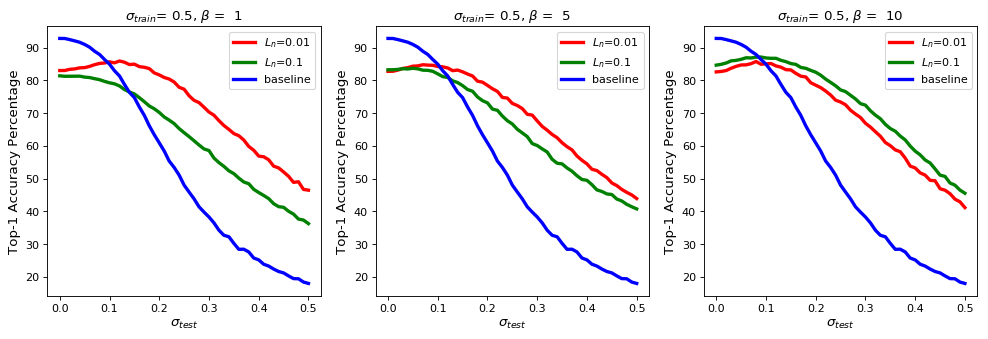}}

\subfloat[Preresnet-20 trained with $\sigma_{train}=0.5$]{\includegraphics[scale=0.4]{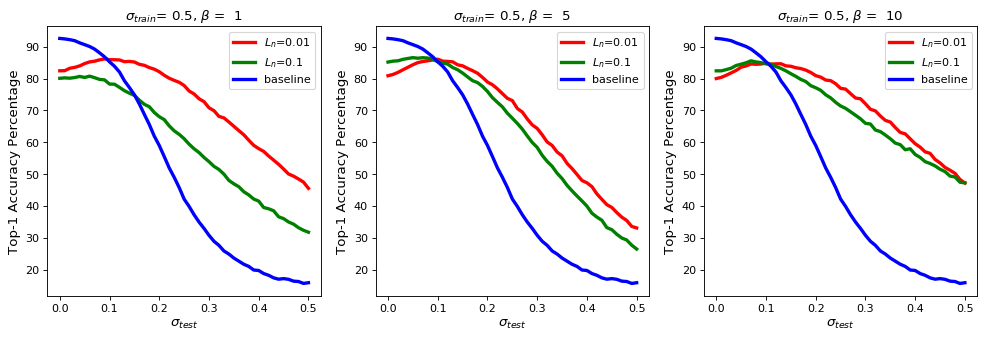}}
\caption{Plots of the top-1 CIFAR-10 test accuracies for $\sigma_{train}=0.5$}
\label{cifar_figure_05}
\end{figure}

\begin{figure}[ht]
\centering
\includegraphics[width=5in]{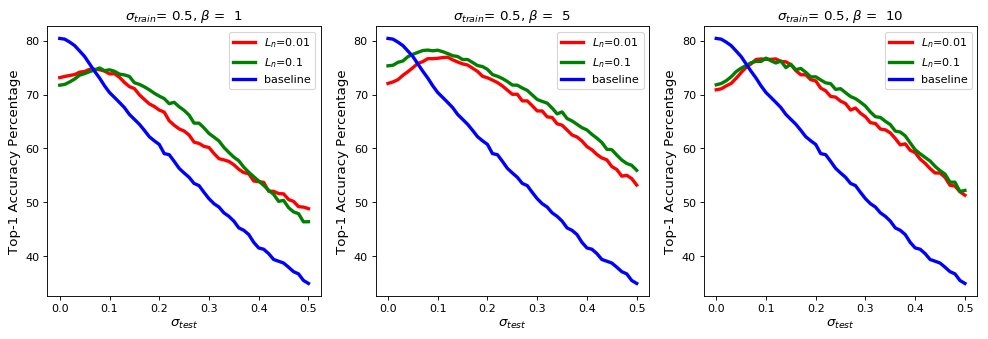}
\caption{Plots of the top-1 STL-10 test accuracies for $\sigma_{train}=0.5$}
\label{stl_figure_05}
\end{figure}


\clearpage
\bibliographystyle{unsrt}
\bibliography{references}

\end{document}